\documentclass[letterpaper]{IEEEtran}
\usepackage{latexsym,,amssymb,amsmath,graphicx,epsf,cite,bbm,float}
\usepackage{ifpdf}
\usepackage{epstopdf}

\usepackage{algorithm,algorithmic}
\usepackage{amsmath,amssymb,bm}
\usepackage{amsfonts,dsfont,color,bbm,subcaption}

\usepackage[linguistics]{forest}

\def\ninept{\def\baselinestretch{1}}
\ninept

\newcommand{\be}{\begin{equation}}
\newcommand{\ee}{\end{equation}}
\newcommand{\bea}{\begin{eqnarray}}
\newcommand{\eea}{\end{eqnarray}}

\newcommand{\MB}{\left[\begin{array}}
\newcommand{\ME}{\end{array}\right]}

\newcommand{\ei}{\end{itemize}}
\newcommand{\bi}{\begin{itemize}}

\newcommand{\X}{\mathcal{X}}

\newcommand{\Z}{\mathcal{Z}}

\newcommand{\A}{\mathcal{A}}
\newcommand{\B}{\mathcal{B}}

\usepackage{amsmath,amsthm}
\usepackage{hyperref}

\DeclareMathOperator*{\argmin}{arg\,min}

\newtheorem{theorem}{Theorem}

\newtheorem{lemma}[]{Lemma}

\newtheorem{definition}[]{Definition}

\newtheorem{assumption}[]{Assumption}

\newcommand{\rulesep}{\unskip\ \vrule\ }
\AtBeginDocument{
\addtolength{\abovedisplayskip}{-0.4ex}
\addtolength{\abovedisplayshortskip}{-0.4ex}
\addtolength{\belowdisplayskip}{-0.4ex}
\addtolength{\belowdisplayshortskip}{-0.4ex}
\addtolength{\belowcaptionskip}{-3.6ex}
}

\begin{document}

\title{Optimally Efficient Sequential Calibration of Binary Classifiers to Minimize Classification Error} 
\author{
	\IEEEauthorblockN{Kaan Gokcesu}, \IEEEauthorblockN{Hakan Gokcesu}
}
\maketitle

\begin{abstract}
	In this work, we aim to calibrate the score outputs of an estimator for the binary classification problem by finding an 'optimal' mapping to class probabilities, where the 'optimal' mapping is in the sense that minimizes the classification error (or equivalently, maximizes the accuracy). We show that for the given target variables and the score outputs of an estimator, an 'optimal' soft mapping, which monotonically maps the score values to probabilities, is a hard mapping that maps the score values to $0$ and $1$. We show that for class weighted (where the accuracy for one class is more important) and sample weighted (where the samples' accurate classifications are not equally important) errors, or even general linear losses; this hard mapping characteristic is preserved. We propose a sequential recursive merger approach, which produces an 'optimal' hard mapping (for the observed samples so far) sequentially with each incoming new sample. Our approach has a logarithmic in sample size time complexity, which is optimally efficient.
\end{abstract}

\section{Introduction}
	\subsection{Preliminaries}
	In most prominent detection, estimation, prediction and learning problems \cite{poor_book, cesa_book}, intelligent agents often make decisions under considerable uncertainty (randomness, noise, incomplete data), where they combine features to determine the actions that maximize some utility \cite{russel2010}. The applications are numerous in many fields including decision theory \cite{tnnls4}, control theory \cite{tnnls3}, game theory \cite{tnnls1,chang}, optimization \cite{zinkevich,hazan}, distribution estimation \cite{gokcesu2018density,willems,gokcesu2018anomaly,coding2}, anomaly detection \cite{gokcesu2019outlier,gokcesu2016nested}, signal processing \cite{gokcesu2018semg,ozkan}, prediction \cite{singer,gokcesu2016prediction} and bandits \cite{cesa-bianchi,gokcesu2018bandit}. The outputs of these learning models are designed to discriminate the data patterns and provide accurate probabilities for practical usefulness. Most learning methods produce classifiers that output scores which can be used to rank the samples from the most to the least probable member of a class. However, in many applications, a ranking of probabilities is not enough and accurate estimates are needed. To this end, developing a calibration method for post-processing the output of commonly used classifiers to generate accurate probabilities has become important. A classifier is well-calibrated if the predicted probabilities coincide with the empirical ones. Deviations from perfect calibration are common in practice and vary depending on the classification models \cite{naeini2015obtaining}. Producing well-calibrated probabilities are critical in many areas of science (e.g., which experiment to conduct), medicine (e.g., which therapy to use), business (e.g., which investment to make) etc. In learning problems, obtaining well-calibrated classifiers is crucial not only for decision making, but also for combining \cite{bella2013effect} or comparing \cite{zhang2004naive,jiang2005learning, hashemi2010application} different classifiers. Research on learning well-calibrated models is not as extensive as learning models with high discrimination.
	
	\subsection{Calibration of Classification Models}
	Probability estimates are important when the classification outputs are not used in isolation but are combined with other sources of information for decision-making, such as the outputs of another classifier or example-dependent misclassification costs \cite{zadrozny2001obtaining}. For example, in handwritten character recognition problem, the outputs from the classifier are used as input to a high-level system which incorporates domain information, such as a language model.
	
	\subsubsection{Well-Calibrated Classification} There are two main approaches to obtaining well-calibrated classification models: 
	\begin{itemize}
		\item The first approach is to build a classification model that is intrinsically well-calibrated.  
		\item The second approach is to rely on the existing discriminative classification models and calibrating their output using post-processing. 
	\end{itemize}
	The first approach will restrict the designer of the learning model by requiring major changes in the objective function (e.g, using a different loss function) and could potentially increase the complexity and computational cost of the associated optimization to	learn the model. Whilst the second approach is general, flexible, and it frees the designer from modifying the learning procedure and the associated optimization method \cite{naeini2015obtaining}.	However, this approach has the potential to decrease discrimination while increasing calibration, if care is not taken.
	
	\subsubsection{Post-Processing Calibration} In general, post-processing calibration methods have two main applications. 
	\begin{itemize}
		\item First, they can be used to convert the outputs of discriminative classification methods with no apparent probabilistic interpretation to posterior class probabilities.
		\item Second, calibration methods can be applied to improve the calibration of a miscalibrated probabilistic model.
	\end{itemize}
	An example to the first application is an SVM model that learns a discriminative model that does not have a direct probabilistic interpretation. In \cite{platt1999probabilistic}, they show that the use of calibration maps SVM outputs to well-calibrated probabilities. An example to the second application is in Naive Bayes (NB) (which is a probabilistic model but its class posteriors are often miscalibrated due to unrealistic independence assumptions \cite{niculescu2005predicting}), where the aim is to improve the calibration without reducing the discrimination. This approach can also work well on calibrating models that are less miscalibrated than NB.
	
	Hence, the collective goal is to map some outputs (estimations) of a learner to suitable probabilities. If the learning method does not overfit the training data, we can use the same data to learn this function. Otherwise, we need to break	the training data into two sets: one for learning the classifier and the other for learning the mapping function.
	
	\subsection{Literature Review}
	In literature, there are various methods to map the outputs of an estimator to probabilities that are well-calibrated, where all of them require a regularization to avoid learning a mapping that does not generalize well to new data (over-fitting). 
	\subsubsection{Parametric Approaches}
	One possible regularization is to impose a parametric shape and use the available data to learn the parameters.	The approach by Platt \cite{platt1999probabilistic} is one such method, which uses a sigmoid to map the outputs into calibrated probabilities. The parameters of the sigmoid are learned in a maximum-likelihood framework (such that the negative log-likelihood is minimized) using a model-trust minimization algorithm \cite{gill2019practical}. The method was originally developed to transform the output of an SVM model into calibrated probabilities, since the relationship between SVM scores and the empirical probabilities appears to be sigmoidal for many datasets. Platt has shown empirically that this faster method yields probability estimates that are at least as accurate as ones obtained by training another SVM specifically for producing the probability estimates. It has also been used to calibrate other type of classifiers \cite{niculescu2005predicting}. One such application is to Naive Bayes, which was proposed by Bennett \cite{bennett2000assessing}. However, the sigmoid shape does not appear to fit Naive Bayes scores as well (in comparison to SVM) for certain datasets \cite{zadrozny2002transforming}. While this approach prevents over-fitting and is computationally efficient; it is restrictive \cite{jiang2012calibrating} (since the scores outputs are directly used whether they are noisy or erroneous).
	
	\subsubsection{Histogram Binning}
	To address the issues of parametric approaches; the less restrictive, non-parametric calibration methods such as the equal frequency histogram binning model (also known as quantile binning or just binning) \cite{zadrozny2001learning, zadrozny2001obtaining} have become popular. In binning, scores are sorted and partitioned into bins of equal size. For each new score in a specific bin, the calibrated probability is estimated as the fraction of the samples of a particular class. Instead of completely trusting the score output values, the main idea is to use outputs that are near each other to compute their probabilities, which is where the regularization comes from. While it has less restrictions and is computationally efficient; the bin boundaries remain fixed over all predictions and there is uncertainty in the optimal number of the bins \cite{zadrozny2002transforming}.
For small or unbalanced datasets, the optimal number of bins may be hard to determine. Moreover, since the size and the position of the bins are chosen arbitrarily, we may fail to produce accurate calibrated probabilities. To this end, there are many extensions/refinements, one of which is ACP \cite{jiang2012calibrating} that derives a confidence interval around each prediction to build the bins. BBQ \cite{naeini2015obtaining} is another one, which addresses the drawbacks by considering multiple binning models and their combination with a Bayesian scoring function \cite{heckerman1995learning}. However, the positions and boundaries of the bins are still selected with equal histogram binning. ABB \cite{naeini2015binary} addresses this by considering Bayesian averaging over all possible binning models induced by the samples, whose main drawback is its complexity (quadratic in the sample size). However, all these approaches do not take advantage of the fact that the input estimator has good discrimination (otherwise, the outputs are better used as features for another classifier).
	
	\subsubsection{Isotonic Regression}
	To address the issues of both parametric and binning approaches, the most commonly used non-parametric classifier calibration method in machine learning has become the isotonic regression \cite{robertson1988order} based calibration (IsoRegC) models \cite{zadrozny2002transforming}, which is an intermediary approach between sigmoid fitting and binning. Isotonic regression is a non-parametric regression, where it assumes the mapping is isotonic (monotonic) based on the ranking imposed by the base estimator from the uncalibrated outputs to the calibrated probabilities, which is where the regularization comes from. If we assume that the estimator ranks samples correctly; the mapping from scores into probabilities is non-decreasing, which can be learned with isotonic regression. A commonly used algorithm in isotonic regression is pair-adjacent violators (PAV) \cite{ayer1955empirical}, whose computational complexity is linear in the number of samples \cite{brunk1972statistical}. To calculate the calibrated probabilities, such algorithms will use more samples in parts of the score space where the estimator ranks them incorrectly, and less samples in parts of the space where the estimator ranks them correctly. We can view an IsoRegC model based on PAV as a binning algorithm since the number, position and size of the bins are chosen according to how well the classifier ranks the samples, i.e., the position of the boundaries are selected by fitting the best monotone approximation to the samples according to the ordering imposed by the estimator \cite{zadrozny2002transforming}.
	Approaches that address the issues of binning by incorporating isotonic regression have also become popular such as ENIR \cite{naeini2016binary}, which utilizes the path algorithm modified pool adjacent violators algorithm (mPAVA) that can find the solution path to a near isotonic regression problem in linearithmic time \cite{tibshirani2011nearly} and combines the predictions made by these models. There is also a variation of the isotonic-regression-based calibration method for predicting accurate probabilities with a ranking loss \cite{menon2012predicting}. Another extension combines the outputs from multiple binary classifiers to obtain calibrated probabilities \cite{zhong2013accurate}.

\subsection{Contributions}

Calibration mappings need to be well-regularized (no over-fitting), less restrictive (no under-fitting), easy to optimize (better modeling) and easy to update (practical use). While all the existing approaches are somewhat well-regularized, each of them has a distinct disadvantage. The parametric approaches are too restrictive; the histogram binning approaches are hard to optimize (number, location, size of the bins); the isotonic regression approaches are hard to update (with new samples). While these methods are ideologically different, we show that an optimal isotonic mapping for the minimization of the classification error is a thresholding function, which coincides with a parametric sigmoid mapping with suitably selected parameters; and histogram binning with suitably selected bin numbers, locations and sizes. We also extend our results to class weighted, sample weighted error, and even general linear losses. We propose a new approach to find the threshold by recursive merger of adjacent sets with its sequential implementation, which can update the threshold with each new incoming sample. Our approach is optimally efficient (logarithmic in sample size for each new sample), which makes it suitable to implement in many applications.
	
\section{Optimal Monotone Transform For Binary Classification is a Thresholding Function}\label{sec:binary}
In this section, we show that, for the problem of minimizing the classification error, an optimal monotone transform on the score values produced by an estimator is a thresholding function, hence, a hard mapping. We start with the formal problem definition.
\subsection{Problem Definition}
	We have $N$ number of samples indexed by $n\in\{1,\ldots,N\}$. For every $n$, 
	\begin{enumerate}
		\item We have the target value $y_n$, which is the binary class of the $n^{th}$ sample, i.e., 
		\begin{align}
		y_n\in\{0,1\}.\label{eq:y}
		\end{align}
		\item We have the output $x_n$ of an estimator, which is the score of the $n^{th}$ sample, i.e.,
		\begin{align}
		x_n\in\overline\Re.\label{eq:x}
		\end{align} 
		\item We map the score outputs $x_n$ of the estimator to probabilities $p_n$ with a mapping function $C(\cdot)$, i.e., 
		\begin{align}
		p_n\triangleq C(x_n)\in[0,1].\label{eq:p}
		\end{align}
		\item We choose the mapping $C(\cdot)$ such that it is a monotone transform (monotonically nondecreasing), i.e.,
		\begin{align}
		C(x_n)\geq C(x_{n'}) \text{ if } x_n> x_{n'}.\label{eq:C}
		\end{align} 
	\end{enumerate}
	For the given setting above, we have the following problem definition for the minimization of the binary classification error.
	\begin{definition}\label{def:problem}
		For a given set of score outputs $\{x_n\}_{n=1}^N$ and target values $\{y_n\}_{n=1}^N$, the minimization of the binary classification error problem is given by
		\begin{align*}
		\argmin_{C(\cdot)\in\Omega}\sum_{n=1}^{N}y_n(1-C(x_n))+(1-y_n)C(x_n),
		\end{align*}
		where $\Omega$ is the class of all univariate monotonically nondecreasing functions that map from $\overline\Re$ to the interval $[0,1]$.
	\end{definition}
	The problem in \autoref*{def:problem} aims to minimize the classification error. Without loss of generality and problem definition, we make the following assumptions.
	\begin{assumption}\label{ass:monotone}
	Let $\{x_n\}_{n=1}^N$ be in an ascending order, i.e.,
	\begin{align*}
	x_{n-1}\leq x_n,
	\end{align*}
	since if they are not, we can simply order the scores $x_n$ and acquire the corresponding $y_n$ target variables.
	\end{assumption}
	\begin{assumption}\label{ass:dummy}
		We have the following two sample pairs
		\begin{align*}
		y_1=0,\text{     } x_1=-\infty && y_N=1, \text{     }x_N=\infty,
		\end{align*}
		since, otherwise, we can arbitrarily add these dummy samples, which does not change the result of the original problem.
	\end{assumption}
	
	Next, we show why the optimal monotone transform $C(\cdot)$ on the scores $x_n$ is a thresholding function.
	
	\subsection{Optimality of Thresholding for Classification Error}\label{sec:optimal}
	Let us have a minimizer $C^*(\cdot)$ for \autoref*{def:problem} given by
	\begin{align}
		C^*(x_n)=p^*_n, &&n\in\{1,\ldots,N\},\label{eq:prob}
	\end{align} 
	where $p^*_n$ be the corresponding probabilities.
	\begin{lemma}\label{lem:sample}
	If $C^*(\cdot)$ is a minimizer for \autoref*{def:problem}, then
	\begin{align*}
	C^*(x_n)=p^*_n=
	\begin{cases}
	p^*_{n+1}, & y_n=1\\
	p^*_{n-1}, & y_n=0
	\end{cases},
	\end{align*}
	for $n=\{2,\ldots,N-1\}$, where $p^*_1=0$ and $p^*_N=1$, which are the dummy samples from \autoref*{ass:dummy}.
	\end{lemma}
	\begin{proof}
	   The proof is straightforward since increasing $p^*_n$ decreases \autoref*{def:problem}, if $y_n=1$ (opposite for $y_n=0$).
	\end{proof}
	Hence, there are groups of samples with the same $p^*_n$. Let there be $I$ groups, where the group $i\in\{1,\ldots,I\}$ cover the samples $n\in\{n_i+1,\ldots,n_{i+1}\}$ ($n_1=0$ and $n_{I+1}=N$).
	
	\begin{lemma}\label{lem:group}
		If $C^*(\cdot)$ is a minimizer for \autoref*{def:problem}, then
		\begin{align*}
		C^*(\{x_n\}_{n=n_i+1}^{n_{i+1}})\triangleq p^*_i=	
		\begin{cases}
		p^*_{i+1},&  \sum_{n=n_i+1}^{n_{i+1}}y_n>\frac{n_{i+1}-n_{i}}{2}\\
		p^*_{i-1},& \sum_{n=n_i+1}^{n_{i+1}}y_n<\frac{n_{i+1}-n_{i}}{2}
		\end{cases},
		\end{align*}
		for $i=\{2,\ldots,I-1\}$. Moreover, $p^*_1=0$ and $p^*_I=1$.
	\end{lemma}
	\begin{proof}
		The proof follows \autoref*{lem:sample} and its proof.
	\end{proof}
	Thus, if the $i^{th}$ group's targets are mostly $1$, then $p^*_i=p^*_{i+1}$ (similarly, $p^*_i=p^*_{i-1}$, if mostly $0$). Only if half of the targets are $1$, then $p^*_i$ may not necessarily be equal to $p^*_{i-1}$ or $p^*_{i+1}$.
	
	\begin{lemma}\label{lem:terminate}
		If $C^*(\cdot)$ is a minimizer for \autoref*{def:problem} with $I^*$ uniquely mapped groups (with distinct probabilities) such that
		\begin{align*}
			C^*(\{x_n\}_{n=n_i+1}^{n_{i+1}})=p^*_i&&|&&p^*_i\neq p^*_{i'}, \forall i\neq i'\in\{1,\ldots,I^*\},
		\end{align*}
		where, for $i\in\{2,\ldots,I^*-1\}$, half of the samples' target ($y_n$) are $1$. The first and last group are mostly $0$ and $1$ respectively.
	\end{lemma}
		\begin{proof}
		The proof follows from \autoref*{lem:group}. If the $i^{th}$ group's target variables are not equally distributed then $p^*_i$ is either $p^*_{i-1}$ or $p^*_{i+1}$. Hence, the probabilities will not be unique.
		\end{proof} 

	From \autoref*{lem:terminate}, we reach the following theorem. 
	\begin{theorem}\label{thm:threshold}
		There exist a monotone transform $C^*(\cdot)\in\Omega$ (where $\Omega$ is the class of all univariate monotonically nondecreasing functions) that minimizes \autoref*{def:problem} such that
		\begin{align*}
			C^*(\{x_n\}_{n=1}^{{\tau}})\triangleq p^*_0=0,&&
			C^*(\{x_n\}_{n=\tau+1}^{{N}})\triangleq p^*_1=1,
		\end{align*}
		for some $\tau\in\{1,\ldots,N-1\}$.
		\begin{proof}
			The proof follows from \autoref*{lem:terminate}. If there are $I^*$ uniquely mapped groups; the first group is mostly $0$, the last group is mostly $1$, and the other groups are equally distributed. For the groups with equally distributed targets, every probability will produce the same loss and $0$ or $1$ are also optimal. Thus, every $\tau=n_i$ for $i\in\{2,\ldots,I^*\}$ produces the minimum classification error, which concludes the proof.
		\end{proof}
	\end{theorem}
	In this section, we have shown that there exists an optimal classifier $C^*(\cdot)$ (with monotone mapping) that minimizes the classification error in \autoref*{def:problem}, which is also a thresholding function on the score variables $x_n$, i.e., a hard classifier. In the next section, we show that a thresholding function again minimizes the variant of the problem in \autoref*{def:problem}.
	
	\subsection{Extension to Class Weighted Error}
	In this section, we prove that the optimal monotone transform that minimizes the class weighted classification error is again a thresholding function. 
	The setting in \eqref{eq:y}, \eqref{eq:x}, \eqref{eq:p}, \eqref{eq:C} remains the same. We also make the same two assumptions in \autoref*{ass:monotone} and \autoref*{ass:dummy}.
	The revised version of the problem in \autoref*{def:problem} is as the following.
	\begin{definition}\label{def:problemR}
		For $\{x_n\}_{n=1}^N$, $\{y_n\}_{n=1}^N$, the minimization of the weighted class binary classification error is given by
		\begin{align*}
		\argmin_{C(\cdot)\in\Omega}\sum_{n=1}^{N}\alpha y_n(1-C(x_n))+(1-y_n)C(x_n),
		\end{align*}
		where $\Omega$ is the class of all univariate monotonically increasing functions that map from $\overline\Re$ to the interval $[0,1]$ and $\alpha>0$ is the relative classification error weight of the binary class $1$.
	\end{definition}
	Let us again assume that there exists a monotone transform $C^*(\cdot)$ that minimizes \autoref*{def:problemR} with its corresponding probabilities for each sample as in \eqref{eq:prob}.
	\autoref*{lem:sample} directly holds true similarly and we end up with $I$ groups of consecutive samples (with the corresponding intervals) that map to the same probability. \autoref*{lem:group} is modified as the following.
	
	\begin{lemma}\label{lem:groupR}
		If $C^*(\cdot)$ is a minimizer for \autoref*{def:problemR}, then
		\begin{align*}
		C^*(\{x_n\}_{n=n_i+1}^{n_{i+1}})\triangleq p^*_i=	
		\begin{cases}
		p^*_{i+1},&  \sum\limits_{n=n_i+1}^{n_{i+1}} y_n>\frac{(n_{i+1}-n_{i})}{\alpha+1},\\
		p^*_{i-1},& \sum\limits_{n=n_i+1}^{n_{i+1}} y_n<\frac{(n_{i+1}-n_{i})}{\alpha+1},
		\end{cases}
		\end{align*}
		for $i=\{2,\ldots,I-1\}$. Moreover, $p^*_1=0$ and $p^*_I=1$.
	\end{lemma}
	\begin{proof}
		The proof follows from \autoref*{lem:group} and its proof. If $\alpha=1$ (i.e., the classification errors for different classes are equally weighted), the statement becomes equal to \autoref*{lem:group}.
	\end{proof}
	\autoref*{lem:groupR} states that if more than $\frac{1}{\alpha+1}$ of the $i^{th}$ group's target variables are $1$, then $p^*_i=p^*_{i+1}$ (similarly, $p^*_i=p^*_{i-1}$, if more than $\frac{1}{\alpha+1}$ are $0$). Moreover, we can see that if the number of $1$'s in the $i^th$ group is exactly $\frac{1}{\alpha+1}$, then $p^*_i$ may not necessarily be equal to $p^*_{i-1}$ or $p^*_{i+1}$.
	
	\begin{lemma}\label{lem:terminateR}
		If $C^*(\cdot)$ is an optimal classifier from \autoref*{def:problemR} with $I^*$ uniquely mapped sample groups (each group having a different probability) such that
		\begin{align*}
		C^*(\{x_n\}_{n=n_i+1}^{n_{i+1}})=p^*_i,&&|&&		p^*_i\neq p^*_{i'}, \forall i\neq i'\in\{1,\ldots,I^*\},
		\end{align*}
		where, for $i\in\{2,\ldots,I^*-1\}$, $\frac{1}{\alpha+1}$ of the samples' targets ($y_n$) are $1$. More than $\frac{1}{\alpha+1}$ of the first group's targets are $0$ and more than $\frac{1}{\alpha+1}$ of the last group's targets are $1$.
		\begin{proof}
			The proof is similar to the proof of \autoref*{lem:terminate}
		\end{proof}
	\end{lemma}
	Similar to \autoref*{thm:threshold}, we have the following theorem.
	\begin{theorem}\label{thm:thresholdR}
		There exist an optimal classifier $C^*(\cdot)\in\Omega$ (where $\Omega$ is the class of all univariate monotonically nondecreasing functions) that minimizes \autoref*{def:problemR} such that
		\begin{align*}
		C^*(\{x_n\}_{n=1}^{{\tau}})\triangleq p^*_0=0,&&C^*(\{x_n\}_{n=\tau+1}^{{N}})\triangleq p^*_1=1,
		\end{align*}
		for some $\tau\in\{1,\ldots,N-1\}$.
		\begin{proof}
			The proof follows from \autoref*{lem:terminateR} and is similar to the proof of \autoref*{thm:threshold}.
		\end{proof}
	\end{theorem}
	
\subsection{Extension to Sample Weighted Error}
In this section, we prove that the optimal monotone transform that minimizes the sample weighted classification error is again a thresholding function. 
The setting in \eqref{eq:y}, \eqref{eq:x}, \eqref{eq:p}, \eqref{eq:C} remains the same. We also make the same two assumptions in \autoref*{ass:monotone} and \autoref*{ass:dummy}.
The revised version of the problem in \autoref*{def:problemR} is as the following.
\begin{definition}\label{def:problemS}
	For $\{x_n\}_{n=1}^N$, $\{y_n\}_{n=1}^N$, the minimization of the weighted class binary classification error is given by
	\begin{align*}
	\argmin_{C(\cdot)\in\Omega}\sum_{n=1}^{N}\beta_n\alpha y_n(1-C(x_n))+\beta_n(1-y_n)C(x_n),
	\end{align*}
	where $\Omega$ is the class of all univariate monotonically increasing functions that map from $\overline\Re$ to the interval $[0,1]$ and $\alpha>0$ is the relative classification error weight of the binary class $1$ and $\beta_n>0$ is the weight of the $n^{th}$ sample.
\end{definition}
Let us again assume that there exists a monotone transform $C^*(\cdot)$ that minimizes \autoref*{def:problemS} with its corresponding probabilities for each sample as in \eqref{eq:prob}.
\autoref*{lem:sample} directly holds true similarly and we end up with $I$ groups of consecutive samples (with the corresponding intervals) that map to the same probability. \autoref*{lem:group} is modified as the following.

\begin{lemma}\label{lem:groupS}
	If $C^*(\cdot)$ is a minimizer for \autoref*{def:problemR}, then
	\begin{align*}
	C^*(\{x_n\}_{n=n_i+1}^{n_{i+1}})\triangleq p^*_i=	
	\begin{cases}
	p^*_{i+1},&  Y_{n_{i}+1}^{n_{i+1}}>\frac{1}{\alpha+1}B_{n_{i}+1}^{n_{i+1}},\\
	p^*_{i-1},& Y_{n_{i}+1}^{n_{i+1}}<\frac{1}{\alpha+1}B_{n_{i}+1}^{n_{i+1}},
	\end{cases}
	\end{align*}
	for $i=\{2,\ldots,I-1\}$, where $Y_{n_{i}+1}^{n_{i+1}}\triangleq\sum_{n=n_i+1}^{n_{i+1}} \beta_ny_n$ and $B_{n_{i}+1}^{n_{i+1}}\triangleq\sum_{n=n_i+1}^{n_{i+1}} \beta_n$. Moreover, $p^*_1=0$ and $p^*_I=1$.
\end{lemma}
\begin{proof}
	The proof follows from \autoref*{lem:groupR} and its proof. If $\beta_n=1$ (i.e., the classification errors for different samples are equally weighted), the statement becomes equal to \autoref*{lem:groupR}.
\end{proof}
\autoref*{lem:groupR} states that if the total weight of $y_n=1$ in group $i$ is more than $\frac{1}{\alpha+1}$ of the $i^{th}$ group's total weight, then $p^*_i=p^*_{i+1}$ (conversely, $p^*_i=p^*_{i-1}$). If it is exactly $\frac{1}{\alpha+1}$, then $p^*_i$ may not necessarily be equal to $p^*_{i-1}$ or $p^*_{i+1}$.

\begin{lemma}\label{lem:terminateS}
	If $C^*(\cdot)$ is an optimal classifier from \autoref*{def:problemR} with $I^*$ uniquely mapped sample groups (each group having a different probability) such that
	\begin{align*}
	C^*(\{x_n\}_{n=n_i+1}^{n_{i+1}})=p^*_i,&&|&&		p^*_i\neq p^*_{i'}, \forall i\neq i'\in\{1,\ldots,I^*\},
	\end{align*}
	where, for $i\in\{2,\ldots,I^*-1\}$, the total weight of $y_n=1$ in group $i$ is $\frac{1}{\alpha+1}$ of the group's total weight. The total weight of $1$ is less than $\frac{1}{\alpha+1}$ of the first group's total weight and more than $\frac{1}{\alpha+1}$ of the last group's total weight respectively.
	\begin{proof}
		The proof is similar to the proof of \autoref*{lem:terminateR}
	\end{proof}
\end{lemma}
Similar to \autoref*{thm:thresholdR}, we have the following theorem.
\begin{theorem}\label{thm:thresholdS}
	There exist an optimal classifier $C^*(\cdot)\in\Omega$ (where $\Omega$ is the class of all univariate monotonically nondecreasing functions) that minimizes \autoref*{def:problemR} such that
	\begin{align*}
	C^*(\{x_n\}_{n=1}^{{\tau}})\triangleq p^*_0=0,&&C^*(\{x_n\}_{n=\tau+1}^{{N}})\triangleq p^*_1=1,
	\end{align*}
	for some $\tau\in\{1,\ldots,N-1\}$.
	\begin{proof}
		The proof follows from \autoref*{lem:terminateS} and is similar to the proof of \autoref*{thm:thresholdR}.
	\end{proof}
\end{theorem}

	\section{Optimal Monotone Transform for Linear Losses is a Thresholding Function}\label{sec:linear}
	In this section, we generalize the results up to now to show that the optimal monotone transform on our estimations $x_n$ that minimizes a general linear loss game is again a thresholding function. 
	\subsection{Problem Definition}
	We again have samples indexed by  $n\in\{1,\ldots,N\}$, and for every $n$, 
	\begin{enumerate}
		\item We have the score output of an estimator
		\begin{align}
		x_n\in\overline\Re.\label{eq:xL}
		\end{align} 
		\item We have the linear loss 
		\begin{align}
		z_n\in[Z_0,Z_1], && Z_0,Z_1\in\Re, &&Z_0\leq 0\leq Z_1.\label{eq:z}
		\end{align}
		\item We map these score values to values
		\begin{align}
		q_n\triangleq C(x_n)\in[Q_0,Q_1].\label{eq:q}
		\end{align}
		\item The mapping $C(\cdot)$ is monotonically nondecreasing, i.e.,
		\begin{align}
		C(x_n)\geq C(x_{n'}) \text{ if } x_n> x_{n'}.\label{eq:CL}
		\end{align} 
	\end{enumerate}
	For this new setting, we have the following problem definition.
	\begin{definition}\label{def:problemL}
		For $\{x_n\}_{n=1}^N$, $\{z_n\}_{n=1}^N$, the minimization of the linear loss is given by
		\begin{align*}
		\argmin_{C(\cdot)\in\Omega_q}\sum_{n=1}^{N}z_nC(x_n),\label{eq:problemL}
		\end{align*}
		where $\Omega_q$ is the class of all univariate monotonically increasing functions that map to the interval $[Q_0,Q_1]$.
	\end{definition}
	We point out that the problem in \autoref*{def:problemL} fully generalizes the problems in \autoref*{def:problem}, \ref*{def:problemR} and \ref*{def:problemS}. Specifically,
	\begin{itemize}
		\item For \autoref*{def:problem}, \ref*{def:problemR} and \ref*{def:problemS}, we have
		\begin{align*}
			[Q_0,Q_1]=[0,1]
		\end{align*}
		\item For \autoref*{def:problem}, we have
		\begin{align*}
			z_n=1-2y_n, &&[Z_0,Z_1]=[-1,1]
		\end{align*} 
		\item For \autoref*{def:problemR}, we have
		\begin{align*}
			z_n=1-(\alpha+1)y_n, &&[Z_0,Z_1]=[-\alpha,1]
		\end{align*}
		\item For \autoref*{def:problemS}, we have
		\begin{align*}
			z_n=\beta_n-\beta_n(\alpha+1)y_n, &&[Z_0,Z_1]=[-\alpha\beta,\beta],
		\end{align*}
		where $\beta\geq\beta_n$, $\forall n$.
	\end{itemize}
	
	We make the same assumption as in \autoref*{ass:monotone}. We modify the assumption in \autoref*{ass:dummy} as the following.
	\begin{assumption}\label{ass:dummyL}
		We have the following two sample pairs
		\begin{align*}
		z_1=Z_1,\text{     } x_1=-\infty && z_N=Z_0, \text{     }x_N=\infty,
		\end{align*}
		since, otherwise, we can arbitrarily add these dummy samples, which does not change the result of the original problem.
	\end{assumption}
	Next, we show why the optimal monotone transform $C(\cdot)$ on the scores $x_n$ is a thresholding function for the general linear losses.
	
	\subsection{Optimality of Thresholding for Linear Losses}
	Let us again assume that there exists an optimal monotone transform $C^*(\cdot)$ in $\Omega_q$ that minimizes \autoref*{def:problemL}, where $\Omega_q$ is the class of all univariate monotonically nondecreasing functions that map to the interval $[Q_0,Q_1]$.
	
	We have a similar result to \autoref*{lem:sample} as the following
		\begin{lemma}\label{lem:sampleE}
		If $C^*(\cdot)$ is an optimal classifier for \autoref*{def:problemL}, then
		\begin{align}
		C^*(x_n)=q^*_n=
		\begin{cases}
		q^*_{n+1}, & z_n<0\\
		q^*_{n-1}, & z_n>0
		\end{cases},
		\end{align}
		for $n=\{2,\ldots,N-1\}$. $q^*_1=Q_0$ and $q^*_N=Q_1$, which are the dummy samples from \autoref*{ass:dummyL}.
	\end{lemma}
	\begin{proof}
		The proof is similar to the proof of \autoref*{lem:sample}.
	\end{proof}
	Moreover, the modified version of \autoref*{lem:groupS} is as follows.
	\begin{lemma}\label{lem:groupL}
		If $C^*(\cdot)$ is an optimal classifier from \autoref*{def:problemR}, then
		\begin{align}
		C^*(\{x_n\}_{n=n_i+1}^{n_{i+1}})\triangleq q^*_i=	
		\begin{cases}
		q^*_{i+1},&  \sum_{n=n_i+1}^{n_{i+1}} z_n<0,\\
		q^*_{i-1},& \sum_{n=n_i+1}^{n_{i+1}} z_n>0,
		\end{cases}
		\end{align}
		for $i=\{2,\ldots,I-1\}$. Moreover, $q^*_1=Q_0$ and $q^*_I=Q_1$.
	\end{lemma}
	\begin{proof}
		The proof follows from \autoref*{lem:groupS} and also similar to its proof.
	\end{proof}
	\autoref*{lem:groupL} states that if the summation of the loss weights ($z_n$) of the $i^{th}$ group is negative, then $q^*_i=q^*_{i+1}$ (similarly, $q^*_i=q^*_{i-1}$, if positive). Moreover, we can see that if the summation of the loss weights in the $i^th$ group is exactly $0$, then $q^*_i$ may not necessarily be equal to $q^*_{i-1}$ or $q^*_{i+1}$. Similarly the modified \autoref*{lem:terminateR} is next.
		\begin{lemma}\label{lem:terminateL}
		If $C^*(\cdot)$ is an optimal classifier for \autoref*{def:problemL} with $I^*$ uniquely mapped sample groups (each group having a different mapped value) such that
		\begin{align}
		C^*(\{x_n\}_{n=n_i+1}^{n_{i+1}})=q^*_i,
		\end{align}
		and
		\begin{align}
		q^*_i\neq q^*_{i'}, \text{\enspace\enspace}\forall i,i'\in\{1,\ldots,I^*\}, i\neq i',
		\end{align}
		then for $i\in\{2,\ldots,I^*-1\}$, the summation of the loss weights $z_n$ of the $i^{th}$ group is exactly $0$. The summation of the loss weights of the first group is positive and the last group's is negative.
		\begin{proof}
			The proof is similar to the proof of \autoref*{lem:terminateS}.
		\end{proof}
	\end{lemma}
	
	Similar to \autoref*{thm:threshold}, \ref*{thm:thresholdR} and \ref*{thm:thresholdS}, we have the following theorem.
	\begin{theorem}\label{thm:thresholdL} 
		There exist an optimal classifier $C^*(\cdot)\in\Omega_q$ (where $\Omega_q$ is the class of all univariate monotonically increasing functions that map to $Q_0,Q_1$) that minimizes \autoref*{def:problemL} such that
		\begin{align}
		C^*(\{x_n\}_{n=1}^{{\tau}})\triangleq q^*_0=Q_0,\\
		C^*(\{x_n\}_{n=\tau+1}^{{N}})\triangleq q^*_1=Q_1,
		\end{align}
		for some $\tau\in\{1,\ldots,N-1\}$.
		\begin{proof}
			The proof follows from \autoref*{lem:terminateL} and is similar to the proof of \autoref*{thm:thresholdS}.
		\end{proof}
	\end{theorem}

\section{Finding an Optimal Threshold for Linear Loss}\label{sec:finding}
In this section, we propose algorithms that can find an optimal monotone transform, or equivalently, an optimal threshold for the problem in \autoref*{def:problemL}. 

\begin{definition}\label{def:A}
	Let us have an ordered set of some samples $\X=\{x_1,x_2,\ldots,x_n\}$ and their corresponding linear losses $\Z=\{z_1,z_2,\ldots,z_n\}$.  which is represented by the set
	\begin{align*}
		\A=\{\X,\Z\}.
	\end{align*} 
\end{definition}
Given a set $\A$, its solution is summarized as the following.
\begin{definition}\label{def:B}
	For a given set $\A$ as in \autoref*{def:A}, let an optimal threshold be between $x_k<\tau_*<x_{k+1}$ for $\A$, such that the two adjacent sets $\X_0=\{x_1,\ldots,x_k\}$ and $\X_1=\{x_{k+1},\ldots,x_n\}$ are mapped to $Q_0$ and $Q_1$ respectively. We define the auxiliary set with
	\begin{align*}
	\B=\{X_0,X_1,L_0,L_1\},
	\end{align*} 
	where $X_0=x_k$, $X_1=x_{k+1}$ are the threshold samples and $L_0=\sum_{i=1}^{k}z_i$ and $L_1=\sum_{i=k+1}^{n}z_i$ are the corresponding cumulative linear losses. 
\end{definition}
The set $\B$ in \autoref*{def:B} completely captures the solution and its corresponding cumulative loss. Given the set $\B$, our monotone transform is $C(x\geq X_1)=Q_1$ and $C(x\leq X_0)=Q_0$, with the resulting cumulative loss $L=Q_0L_0+Q_1L_1$.

\subsection{Brute Force Approach: Batch Optimization in $O(N^2)$}\label{sec:brute}
	For a given $\A$ as in \autoref*{def:A}, to find a $\B$ as in \autoref*{def:B}, the brute force approach is to try all possible $k\in\{1,\ldots,N\}$, i.e.,
	\begin{align}
		\B_k=\left\{x_k,x_{k+1},\sum_{n=1}^{k}z_n,\sum_{n=k+1}^{N}z_n\right\},
	\end{align}
	whose cumulative losses are given by
	\begin{align}
		L_k=Q_0\sum_{n=1}^{k}z_n+Q_1\sum_{n=k+1}^{N}z_n
	\end{align}
	Then, we can choose the $k^*$ with the minimum loss, i.e.,
	\begin{align}
	\B=\B_{k^*}, && k^*\triangleq \argmin_{k} L_{k^*},
	\end{align}
	which takes $O(N^2)$ time since $L_k$ takes $O(N)$ for every $k$.
\subsection{Iterative Approach: Batch Optimization in $O(N)$}\label{sec:iterative}
	If we define $L_{n_1:n_2}=\sum_{n=n_1}^{n_2}z_n$, we see that it has telescoping update rules as 
	\begin{align*}
		L_{n_1-1:n_2}=z_{n_1-1}+L_{n_1:n_2},&& L_{n_1:n_2+1}=L_{n_1:n_2}+z_{n_2+1}.
	\end{align*}
	Thus, starting with $L_{1:1}=z_1$, $L_{N:N}=z_N$; we can calculate $L_{1:n}$, $L_{n:N}$ for $n\in\{1,\ldots,N\}$ in $O(N)$ time. Thus, we can optimize the following
	\begin{align}
		\argmin_\tau \left(L_{1:\tau}Q_0+L_{\tau+1:N}Q_1\right),
	\end{align}
	which results in a linear complexity $O(N)$ algorithm. Since simply reading through $\{x_n\}_{n=1}^N$ takes $O(N)$ time, this algorithm is optimally efficient.

\section{Optimally Efficient Recursive Merger}\label{sec:recursive}
While the approach in \autoref*{sec:iterative} is optimally efficient for finding the threshold in batch optimization, its sequential implementation is abysmal. With every new sample observed (which does not necessarily arrive in order), we need to update $O(N)$ number of losses $L_{n_1:n_2}$. Thus, we almost need to rerun the batch algorithm with every new sample, which results in $O(N^2)$ complexity. When the data $x_n$ is ordered, the threshold can be found is $O(N)$ time as in \autoref*{sec:iterative}. If the data $x_n$ is unordered, we can simply order it in $O(N\log N)$ time, which makes it an at most $O(N\log(N))$ problem. Here, we propose a recursive approach that has a sequential implementation, which has optimal $O(\log(N))$ complexity per sample.
\subsection{Batch Optimization}\label{sec:batch}
To find the optimal threshold $\tau_*$ (or the samples $X_0$ and $X_1$), we implement a recursive algorithm. Our algorithm finds the threshold by merging the sets recursively as follows:
\begin{enumerate}
	\item At the initial stage (i.e., the bottom level), we have the sets $\A^{1,n}=\{\X^{1,n},\Z^{1,n}\}$ and the corresponding $\B^{1,n}$ for $n\in\{1,\ldots,N\}$, where $\X^{1,n}=\{x_n\}$, $\Z^{1,n}=\{z_n\}$.
	\item Starting from the bottom level $k=1$ (initial stage), at every level $k$, we create the sets $\A^{k+1,n}$, $\B^{k+1,n}$ by merging the adjacent sets at $k^{th}$ level. Whenever $\A^{k,i}$ and $\A^{k,i+1}$ are merged (with respective $\B^{k,i}$ and $\B^{k,i+1}$), we create the corresponding $\A^{k+1,j}$ and $\B^{k+1,j}$. If there is nothing to merge with a $(\A^{k,i}$, $\B^{k,i})$ pair, they are moved up one level, i.e., $\A^{k+1,j}=\A^{k,i}$ and $\A^{k+1,j}=\B^{k,i}$. Note that $i$ is not necessarily equal to or twice of $j$, since $i$, $j$ are the relative indices at the level $k$, $k+1$ respectively. 
\end{enumerate} 

This algorithm finds an optimal threshold by merging. In the following sections, efficient update of $\B$ is shown.

\subsection{Sequential Update}\label{sec:sequential}
Let the samples $x_n$ come sequentially and in an unordered fashion, and our goal is to find the best splitting (thresholding) with the samples observed so far. While a fresh run of the batch algorithm whenever a new sample comes can find the threshold, it is not efficient. However, we can implement this algorithm in an efficient manner as the following:
\begin{enumerate}
	\item Suppose the batch algorithm is run over the past samples observed, which merges the sets at each level recursively. Whenever a new sample comes, we can just update the necessary intermediate sets. 
	\item Whenever there is a new set (middle) $\A_m$ between an already combined pair of left $\A_l$ and right $\A_r$ sets at an arbitrary level $k$, combine $\A_m$ with $\A_l$ (left-justified bias) and move up $\A_r$ as itself to the next level $k+1$.
	\item Otherwise, if the new set $\A_m$ is not between a pair of already combined sets at an arbitrary level $k$, move up $\A_m$ as itself to the next level $k+1$.
	\item Whenever, there are two adjacent sets at an arbitrary level $k$ that are gonna move up as themselves, combine them at the next level $k+1$.
	\item Whenever a new combination is done at an arbitrary level $k$, update the subsequent combinations at $k+1$.
\end{enumerate} 

\subsection{Working Example}
In this section, we give a simple example to elaborate on our algorithm. Suppose we sequentially receive the following sample pairs in order for output and loss $(x_n,z_n)$: $\{(1,1),(8,-2),(5,3),(4,-4),(6,5),(3,-6),(7,7),(2,-8)\}$. We update the sets $\B$ as in \autoref*{def:B} as the following, where the top and the bottom pairs are $(X_0,X_1)$ and $(L_0,L_1)$.

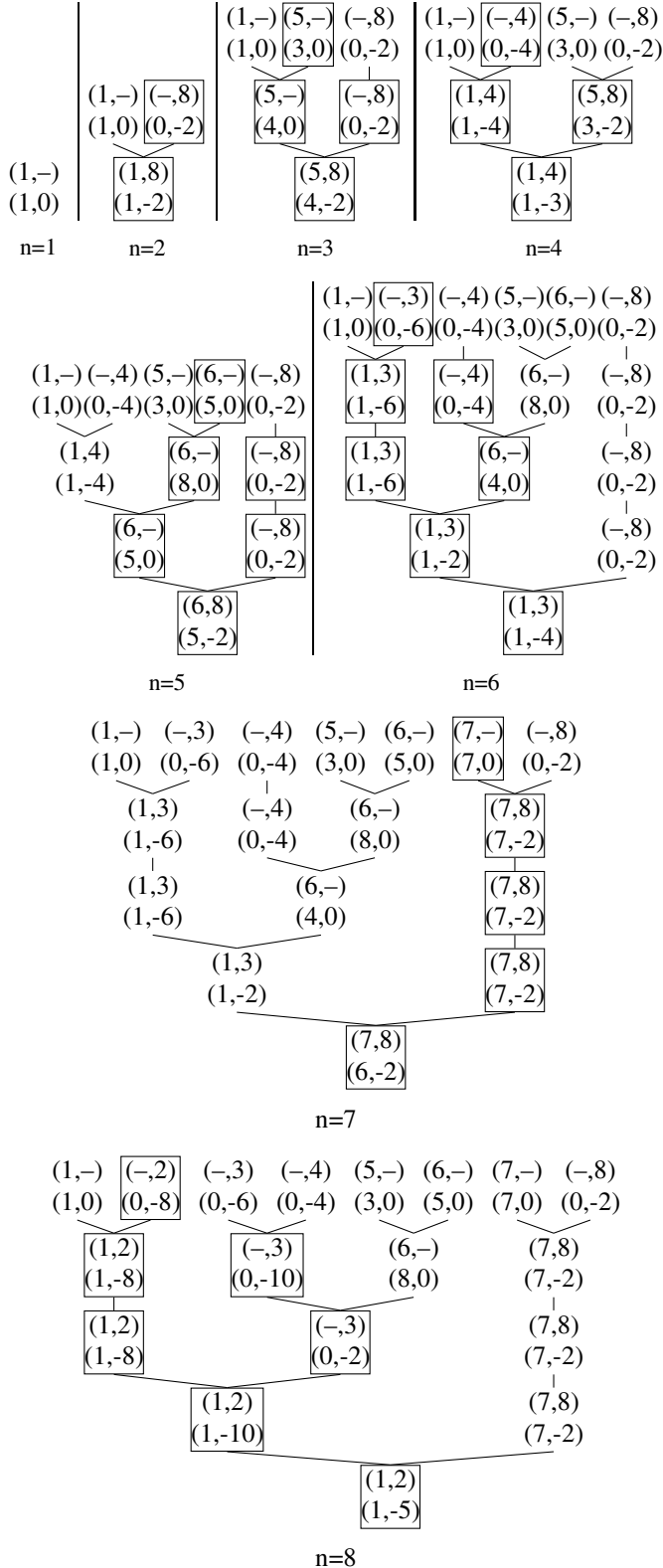
\begin{figure}[h!]
	\centering 
	\begin{subfigure}[t]{.09\columnwidth}
		\begin{forest}
			for tree={
				grow=north,
				s sep=2,inner sep=0pt,l=0pt,l sep=5pt
			}
			[{(1,--)\\(1,0)}
			]
		\end{forest}
		\caption{n=1}
	\end{subfigure}
	\rulesep 
	\begin{subfigure}[t]{.18\columnwidth}
		\begin{forest}
			for tree={
				grow=north,
				s sep=2,inner sep=0pt,l=0pt,l sep=5pt
			}
			[{(1,8)\\(1,-2)},draw	[{(--,8)\\(0,-2)},draw
			]
			[{(1,--)\\(1,0)}
			]
			]
		\end{forest}
		\caption{n=2}
	\end{subfigure}
	\rulesep
	\begin{subfigure}[t]{.27\columnwidth}
		\begin{forest}
			for tree={
				grow=north,
				s sep=2,inner sep=0pt,l=0pt,l sep=5pt
			}
			[{(5,8)\\(4,-2)},draw	[{(--,8)\\(0,-2)},draw	[{(--,8)\\(0,-2)}
			]
			]
			[{(5,--)\\(4,0)},draw	[{(5,--)\\(3,0)},draw
			]
			[{(1,--)\\(1,0)}
			]
			]
			]
		\end{forest}
		\caption{n=3}
	\end{subfigure} 
	\rulesep
	\begin{subfigure}[t]{.36\columnwidth}
		\begin{forest}
			for tree={
				grow=north,
				s sep=2,inner sep=0,l=0,l sep=5pt
			}
			[{(1,4)\\(1,-3)},draw[{(5,8)\\(3,-2)},draw	[{(--,8)\\(0,-2)}
			]
			[{(5,--)\\(3,0)}
			]]
			[{(1,4)\\(1,-4)},draw[{(--,4)\\(0,-4)},draw
			]
			[{(1,--)\\(1,0)}
			]]]
		\end{forest}
		\caption{n=4}
	\end{subfigure} 	
\end{figure}

\begin{figure}[h!] 
	\centering   
	\begin{subfigure}[t]{.41\columnwidth}
		\begin{forest}
			for tree={
				grow=north,
				s sep=0,inner sep=0pt,l=0pt,l sep=5pt
			}
			[{(6,8)\\(5,-2)},draw
			[{(--,8)\\(0,-2)},draw[{(--,8)\\(0,-2)},draw[{(--,8)\\(0,-2)}]]]
			[{(6,--)\\(5,0)},draw[{(6,--)\\(8,0)},draw	[{(6,--)\\(5,0)},draw
			]
			[{(5,--)\\(3,0)}
			]]
			[{(1,4)\\(1,-4)}[{(--,4)\\(0,-4)}
			]
			[{(1,--)\\(1,0)}
			]]]]
		\end{forest}
		\caption{n=5}
	\end{subfigure}
	\rulesep
	\begin{subfigure}[t]{.48\columnwidth}
		\begin{forest}
			for tree={
				grow=north,
				s sep=0,inner sep=0pt,l=0pt,l sep=5pt
			}
			[{(1,3)\\(1,-4)},draw
			[{(--,8)\\(0,-2)}
			[{(--,8)\\(0,-2)}
			[{(--,8)\\(0,-2)}
			[{(--,8)\\(0,-2)}]	
			]
			]
			]
			[{(1,3)\\(1,-2)},draw
			[{(6,--)\\(4,0)},draw
			[{(6,--)\\(8,0)}
			[{(6,--)\\(5,0)}]
			[{(5,--)\\(3,0)}]
			]
			[{(--,4)\\(0,-4)},draw
			[{(--,4)\\(0,-4)}]
			]
			]
			[{(1,3)\\(1,-6)},draw
			[{(1,3)\\(1,-6)},draw
			[{(--,3)\\(0,-6)},draw]
			[{(1,--)\\(1,0)}]
			]
			]
			]
			]
		\end{forest}
		\caption{n=6}
	\end{subfigure}
\end{figure}

\begin{figure}[h!] 
	\centering   
	\begin{forest}
		for tree={
			grow=north,
			inner sep=0pt,l=0pt,l sep=5pt
		}
		[{(7,8)\\(6,-2)},draw
		[{(7,8)\\(7,-2)},draw
		[{(7,8)\\(7,-2)},draw
		[{(7,8)\\(7,-2)},draw
		[{(--,8)\\(0,-2)}][{(7,--)\\(7,0)},draw]	
		]
		]
		]
		[{(1,3)\\(1,-2)}
		[{(6,--)\\(4,0)}
		[{(6,--)\\(8,0)}
		[{(6,--)\\(5,0)}]
		[{(5,--)\\(3,0)}]
		]
		[{(--,4)\\(0,-4)}
		[{(--,4)\\(0,-4)}]
		]
		]
		[{(1,3)\\(1,-6)}
		[{(1,3)\\(1,-6)}
		[{(--,3)\\(0,-6)}]
		[{(1,--)\\(1,0)}]
		]
		]
		]
		]
	\end{forest}
	\caption{n=7}
\end{figure}
\begin{figure}[h!] 
	\centering   
	\begin{forest}
		for tree={
			grow=north,
			inner sep=0pt,l=0pt,l sep=5pt
		}
		[{(1,2)\\(1,-5)},draw
		[{(7,8)\\(7,-2)}
		[{(7,8)\\(7,-2)}
		[{(7,8)\\(7,-2)}
		[{(--,8)\\(0,-2)}][{(7,--)\\(7,0)}]	
		]
		]
		]
		[{(1,2)\\(1,-10)},draw
		[{(--,3)\\(0,-2)},draw
		[{(6,--)\\(8,0)}
		[{(6,--)\\(5,0)}]
		[{(5,--)\\(3,0)}]
		]
		[{(--,3)\\(0,-10)},draw
		[{(--,4)\\(0,-4)}]
		[{(--,3)\\(0,-6)}]
		]
		]
		[{(1,2)\\(1,-8)},draw
		[{(1,2)\\(1,-8)},draw
		[{(--,2)\\(0,-8)},draw]
		[{(1,--)\\(1,0)}]
		]
		]
		]
		]
	\end{forest}
	\caption{n=8}
\end{figure}

\subsection{Complexity Analysis}\label{sec:complexity}
In this section, we prove that our sequential algorithm has logarithmic in time complexity for each new sample.
\begin{lemma}\label{lem:depth}
	The depth $D$ of the recursion ($k\leq D$, $\forall k$) is of order $O(\log N)$, where $N$ is the number of samples.
	\begin{proof}
		Because of the structure of the our sequential implementation, there are no two adjacent sets that moves up the recursion as themselves. Hence, if the number of sets at level $k$ is $N_k$, the number of sets at level $k+1$ is bounded as $N_{k+1}\leq \frac{2}{3}N_{k}+\frac{1}{3}$ (the scenario where the sets alternates between a set that moves up and two sets that combine), which concludes the proof.
	\end{proof}
\end{lemma}

\begin{lemma}\label{lem:traverse}
	We traverse the recursion two times to update the relevant intermediate sets resulting from the new sample.
	\begin{proof}
		Any update (including a new combination) at level $k+1$ will come from a related update at level $k$. With each new sample, we have two updates at the bottom level at most and traverse the recursion for them individually, which results in the lemma.
	\end{proof} 
\end{lemma}

\begin{lemma}\label{lem:update}
	Let us have two sets $\A^{0}$ and $\A^{1}$, whose $\X^{0}$ and $\X^{1}$ are mutually exclusive and adjacent, i.e., $\X^{0}$ and $\X^{1}$ are individually ordered and last element of $\X^{0}$ is less than or equal to the first element of $\X^{1}$. Let the auxiliary sets of $\A^{0}$ and $\A^{1}$ be $\B^0=\{X_0^{0},X_1^{0},L_0^{0},L_1^{0}\}$ and $\B^{1}=\{X_0^{1},X_1^{1},L_0^{1},L_1^{1}\}$. Let $\A=\{\X^0\cup\X^1, \Z^0\cup\Z^1\}$ be the merging of $A^0$ and $A^1$. Then, its auxiliary set is given by
	\begin{align}
	\B=	
	\begin{cases}
	\{X_0^{0},X_1^{0},L_0^{0},L_1^{0}+L_0^1+L_1^1\},&  L_1^0+L_0^1<0\\
	\{X_0^{1},X_1^{1},L_0^0+L_1^0+L_0^1,L_1^{1}\},& L_1^0+L_0^1>0
	\end{cases}.\nonumber
	\end{align}
	\begin{proof}
		Let $\A=\{\X,\Z\}$ (where $\X=\X^0\cup\X^1$ and $\Z=\Z^0\cup\Z^1$) and its auxiliary set $\B=\{X_0,X_1,L_0,L_1\}$. If $\X_0\subseteq \X^0$, 
		then all elements of $\X^1$ will be mapped to $Q_1$.
		Hence, optimization will be done on $\X^0$, which is already given by the separation $\X_0^0$ and $\X_1^0$. Thus, we will have $\X_0=\X_0^0$, $\X_1=\X_1^0\cup\X^1$, and consequently $L_0=L_0^0$, $L_1=L_1^0+L^1$. For the converse, we have $\X_0\supseteq \X^0$, where, this time, all elements of $\X^0$ will be mapped to $Q_0$. Hence, optimization will be done on $\X^1$, which is already given by the separation $\X_0^1$ and $\X_1^1$. Thus, we will have $\X_0=\X^0\cup\X_0^1$, $\X_1=\X_1^1$, and consequently $L_0=L^0+L_0^1$, $L_1=L_1^1$. Between these two choices, the minimum loss results in the former if $L_1^0+L_0^1<0$ and the latter if $L_1^0+L_0^1>0$, which ends the proof.
	\end{proof}
\end{lemma}

\begin{theorem}\label{thm:complexity}
	The sequential update in \autoref*{sec:sequential} will update the optimal threshold in $O(\log N)$ complexity per sample.
	\begin{proof}
		For each new sample, the total number of updates are $O(\log N)$ (from \autoref*{lem:depth} and \ref*{lem:traverse}). Since each update takes $O(1)$ time (from \autoref*{lem:update}), the new threshold is found in $O(\log N)$ time, which concludes the proof.
	\end{proof}
\end{theorem}
Since even the order of the score output $x_n$ of a new sample is found in $O(\log N)$ at best, finding the new threshold in $O(\log N)$ is optimally efficient.

\section{Conclusion}\label{sec:conclusion}
We have studied the problem of finding the optimal monotone transform on the observed score outputs of an estimator, which minimizes the classification error. We have shown that for such problems, an optimal transform is in form of a thresholding function. We have extended our results to include class weighted or sample weighted errors, and even general linear losses. While an optimally efficient $O(N)$ time iterative algorithm is straightforward, it is not easily updatable with new samples. To this end, we have proposed a sequential recursive merger algorithm, which has $O(\log N)$ optimal complexity per new sample, which arrives in an unordered fashion.

\bibliographystyle{ieeetran}
\bibliography{double_bib}
\end{document}